\documentclass{article}

\usepackage[utf8]{inputenc} 
\usepackage[T1]{fontenc}    
\usepackage{hyperref}       
\usepackage{url}            
\usepackage{booktabs}       
\usepackage{amsfonts}       
\usepackage{nicefrac}       
\usepackage{microtype}      

\usepackage[english]{babel}
\usepackage{verbatim}
\usepackage{hyperref}
\usepackage{tabu}
\usepackage{amsmath}
\usepackage{mathrsfs}
\usepackage{amsfonts}
\usepackage{graphicx}
\usepackage{parskip}
\usepackage{color}
\usepackage[]{algorithm2e}
\usepackage{amssymb}
\usepackage[numbers]{natbib}
\newtheorem{theorem}{Theorem}
\newtheorem{Lemma}{Lemma}
\newtheorem{remark}[theorem]{Remark}

\newtheorem{corollary}[Lemma]{Corollary}

\usepackage{amssymb}
\newenvironment{proof}[1][Proof]{\begin{trivlist}
\item[\hskip \labelsep {\bfseries #1}]}{$\blacksquare$\end{trivlist}}

\newcommand*\x[0]{\textbf{x}}

\newcommand*\at[2]{\left.#1\right|_{#2}}

\newcommand*\del[0]{\partial}
\renewcommand*\Re[0]{\mathbb{R}}

\renewcommand*\div[0]{\nabla \cdot}
\newcommand*\ddt[0]{\frac{d}{d t}}
\newcommand*\dds[0]{\frac{d}{d s}}

\newcommand*\tr[0]{\text{tr}}
\newcommand*\KL[2]{\mathcal{KL}\left(#1\|#2\right)}
\newcommand*\lin[1]{\langle #1\rangle}
\newcommand*\E[1]{\mathbb{E}\left[#1\right]}

\newcommand*\D[0]{\mathcal{D}}

\newcommand*\G[0]{\mathcal{G}}
\newcommand*\Pspace[0]{\mathscr{P}}

\newcommand*\Rd[0]{\mathbb{R}^d}

\newcommand\numberthis{\addtocounter{equation}{1}\tag{\theequation}}
\newcommand*\cvec[2]{\begin{bmatrix} #1\\#2\end{bmatrix}}

\newcommand*\lrbb[1]{\left\{#1\right\}}
\newcommand*\lrp[1]{\left(#1\right)}
\newcommand*\lrn[1]{\left\|#1\right\|}
\newcommand*\p[0]{\mathbf{p}}
\renewcommand*\P[0]{\mathbf{P}}
\newcommand*\q[0]{\mathbf{q}}
\newcommand*\Q[0]{\mathbf{Q}}
\newcommand*\g[0]{\mathbf{g}}
\renewcommand*\G[0]{\mathbf{G}}
\renewcommand*\r[0]{\mathbf{r}}

\newcommand*\Ep[2]{\mathbb{E}_{#1}\left[#2\right]}
\newcommand*\Lp[1]{{L^2(#1)}}
\newcommand*\ppi[0]{{\boldsymbol{\pi}}}
\newcommand*\pmu[0]{{\boldsymbol{\mu}}}
\newcommand*\pnu[0]{{\boldsymbol{\nu}}}

\begin{document}

\pagestyle{plain}
\setcounter{page}{1}
\pagenumbering{arabic}
\title{Convergence of Langevin MCMC in KL-divergence}
\author{Xiang Cheng \and Peter Bartlett}
\maketitle

\begin{abstract}
Langevin diffusion is a commonly used tool for sampling from a given distribution. In this work, we establish that when the target density $\p^*$ is such that $\log \p^*$ is $L$ smooth and $m$ strongly convex, discrete Langevin diffusion produces a distribution $\p$ with $\KL{\p}{\p^*}\leq \epsilon$ in $\tilde{O}(\frac{d}{\epsilon})$ steps, where $d$ is the dimension of the sample space.  We also study the convergence rate when the strong-convexity assumption is absent. By considering the Langevin diffusion as a gradient flow in the space of probability distributions, we obtain an elegant analysis that applies to the stronger property of convergence in KL-divergence and gives a conceptually simpler proof of the best-known convergence results in weaker metrics.
\end{abstract}

\section{Introduction}
Suppose that we would like to sample from a density 
$$\p^*(x) = e^{-U(x) + C}$$
where $C$ is the normalizing constant. We know $U(x)$, but we do not know the normalizing constant. This comes up, for example, in variational inference, when the normalization constant is computationally intractable. 

One way to sample from $\p^*$ is to consider the Langevin diffusion:
\begin{align*}
\bar{x}_0 &\sim \bar{\p}_0\\
\numberthis \label{e:exactlangevin0}
d \bar{x}_t &=  -\nabla U(\bar{x}_t) dt + \sqrt{2} dB_t
\end{align*}
Where $\bar{\p}_0$ is some initial distribution and $B_t$ is Brownian motion (see Section \ref{s:definitions}). The stationary distribution of the above SDE is $\p^*$. 

The Langevin MCMC algorithm, given in two equivalent forms in \eqref{e:langevinalgorithm} and \eqref{e:discretelangevin}, is an algorithm based on discretizing \eqref{e:exactlangevin0}.

Previous works have shown the convergence of \eqref{e:discretelangevin} in both total variation distance (\cite{theoretical}, \cite{nonasymptotic}) and 2-Wasserstein distance (\cite{highdim}). The approach in these papers relies on first showing the convergence of \eqref{e:exactlangevin0}, and then bounding the discretization error between \eqref{e:discretelangevin} and \eqref{e:exactlangevin}.

In this paper, our main goal is to establish the convergence of $\p_t$ in \eqref{e:discretelangevin} in $\KL{\p_t}{\p^*}$.  KL-divergence is perhaps the most natural notion of distance between probability distributions in this context, because of its close relationship to maximum likelihood estimation, its interpretation as information gain in Bayesian statistics, and its central role in information theory. Convergence in KL-divergence implies convergence in total variation and 2-Wasserstein distance, thus we are able to obtain convergence rates in total variation and 2-Wasserstein that are comparable to the results shown in (\cite{theoretical}, \cite{nonasymptotic}, \cite{highdim}).

\section{Related Work}

The first non-asymptotic analysis of the discrete Langevin diffusion \eqref{e:discretelangevin} was due to Dalalyan in \cite{theoretical}. This was soon followed by the work by Durmus and Moulines in \cite{nonasymptotic}, which improved upon the results in \cite{theoretical}. Subsequently, Durmus and Moulines also established convergence of \eqref{e:discretelangevin} for the 2-Wasserstein distance in \cite{highdim}. We remark that the proofs of Lemma \ref{l:metricderivativefordiscretizationissmall}, \ref{l:strongconvexityboundexsq} and \ref{l:nonstronglyconvexWassersteindistanceiscontractive}
are essentially taken from \cite{highdim}.

In a slightly different direction from the goals of this paper, Bubeck et al \cite{bubeck} and Durmus et al \cite{moreau} studied variants of \eqref{e:discretelangevin} which work when $-\log \p^*$ is not smooth. This is important, for example, when we want to sample from the uniform distribution over some convex set, so $-\log \p^*$ is the indicator function.

Very recently, Dalalyan et al \cite{dalalyanfriendly} proved the convergence of Langevin Monte Carlo when only stochastic gradients are available.

Our work also borrows heavily from the theory established in the book of Ambrosio, Gigli and Savare \cite{Ambrosio}, which studies the underlying probability distribution $\bar{\p}_t$ induced by \eqref{e:exactlangevin0} as a gradient flow in probability space. This allows us to view \eqref{e:discretelangevin} as a \textit{deterministic} convex optimization procedure over the probability space, with KL-divergence as the objective. This beautiful line of work relating SDEs with gradient flows in probability space was begun by Jordan, Kinderlehrer and Otto \cite{jko}. We refer any interested reader to an excellent survey by Santambrogio in \cite{survey}.

Finally, we remark that the theory in \cite{Ambrosio} has some very interesting connections with the study of normalization flows in \cite{normalizingflow} and \cite{stein}. For example, the tangent velocity of \eqref{e:exactlangevin}, given by $v_t = \nabla \log \p^* - \nabla \log \p_t$, can be thought of as a deterministic transformation that induces a normalizing flow.

\section{Our Contribution}

In this section, we compare the results we obtain with those in \cite{theoretical}, \cite{nonasymptotic} and \cite{highdim}.

Our main contribution is establishing the first nonasymptotic convergence  Kullback-Leibler divergence for \eqref{e:discretelangevin} when $U(x)$ is $m$ strongly convex and $L$ smooth. (see Theorem \ref{t:strongconvexityconvergence}). As a consequence, we also unify the proof of convergence in total variation and $W_2$ as simple corollaries to the convergence in $KL$.

The following table compares the number of iterations of \eqref{e:langevinalgorithm} required to achieve $\epsilon$ error in each of the three quantities according to the analysis of various papers.

\begin{table}[h]
  \caption{Comparison of iteration complexity}

\begin{center}
\begin{tabu} to 0.4\textwidth { | X[l] | X[r] | X[r] | X[r] | }
 \hline
   & $TV$ & $W_2$ & $KL$\\
 \hline
 \cite{theoretical}, \cite{nonasymptotic}  & $\tilde{O}(\frac{d}{\epsilon^2})$  & -  & -\\
 \hline
 \cite{highdim}  & $\tilde{O}(\frac{d}{\epsilon^2})$  & $\tilde{O}(\frac{d}{\epsilon^2})$  & - \\
 \hline
 this paper  & $\tilde{O}(\frac{d}{\epsilon^2})$  & $\tilde{O}(\frac{d}{\epsilon^2})$  & $\tilde{O}(\frac{d}{\epsilon})$\\
 \hline
\end{tabu}
\end{center}
\end{table}

In Section \ref{s:nonsc}, we also state a convergence result for when $U$ is not strongly convex. The corollary for convergence in total variation has a better dependence on the dimension than the corresponding result in \cite{theoretical}, but a worse dependence on $\epsilon$. 

\section{Definitions}
\label{s:definitions}
We denote by $\Pspace(\Re^d)$ the space of all probability distributions over $\Rd$. In the rest of this paper, only distributions with densities wrt the Lebesgue measure will appear (see Lemma \ref{l:regularityofdensitiesandcurves}), both in the algorithm and in the analysis. With abuse of notation, we use the same symbol (e.g. $\p$) to denote both the probability distribution and its density wrt the Lebesgue measure. 

We let $B_t$ be the d-dimensional Brownian motion. 

Let $\p^*$ be the target distribution such that $U(x) = -\log \p^*(x) + C$ has $L$ Lipschitz continuous gradients and $m$ strong convexity, i.e. for all $x$:
$$mI\preceq \nabla^2U(x) \preceq LI$$
For a given initial distribution $\bar{\p}_0$, the \textbf{Exact Langevin Diffusion} is given by the following stochastic differential equation (recall $U(x) - \log \p^*(x)$):
\begin{align*}
\bar{x}_0 &\sim \bar{\p}_0\\
\numberthis \label{e:exactlangevin}
d \bar{x}_t &=  -\nabla U(\bar{x}_t) dt + \sqrt{2} dB_t
\end{align*}
(This is identical to \eqref{e:exactlangevin0}, restated here for ease of reference.)
For a given initial distribution $\p_0$, and for a given stepsize $h$, the \textbf{Langevin MCMC Algorithm} is given by the following:
\begin{align*}
u^0&\sim \p_0\\
u^{i+1} &= u^i - h \cdot \nabla U(u^i) + \sqrt{2h} \xi^i
\numberthis \label{e:langevinalgorithm}
\end{align*}
Where $\xi^i \overset{iid}{\sim} N(0,1)$.

For a given initial distribution $\p_0$ and stepsize $h$, the \textbf{Discretized Langevin Diffusion} is given by the following SDE:
\begin{align*}
\numberthis \label{e:discretelangevin}
x_0 &\sim \p_0\\
d x_t &= -\nabla U(x_{\tau(t)})dt + \sqrt{2} d B_t\\
\text{Let } \p_t & \text{ denote the distribution of } x_t
\end{align*}
Where $\tau(t) \triangleq \lfloor \frac{t}{h}\rfloor \cdot h$ (note that $\tau(t)$ is parametrized by $h$). It is easily verified that for any $i$, $x_{ih}$ from \eqref{e:discretelangevin} is equivalent to $u^i$ in \eqref{e:langevinalgorithm}. Note that the difference between \eqref{e:exactlangevin} and \eqref{e:discretelangevin} is in the drift term: one is $\nabla U(\bar{x}_t)$, the other is $\nabla U (x_{\tau(t)})$

\textbf{For the rest of this paper, we will use $\p_t$ to exclusively denote the distribution of $x_t$ in \eqref{e:discretelangevin}.}

\textbf{We assume  without loss of generality that 
$$\arg\min_x U(x) = 0$$, and that 
$$U(0)=0$$.} (We can always shift the space to achieve this, and the minimizer of $U$ is easy to find using, say, gradient descent.)

For the rest of this paper, we will let 
$$F(\pmu) = 
\begin{cases}
\int \pmu(x) \log\left(\frac{\pmu(x)}{\p^*(x)}\right) dx, & \text{if $\pmu$ has density wrt}\\
&\text{Lebesgue measure} \\
\infty & \text{else}
    \end{cases}
$$
be the KL-divergence between $\pmu$ and $\p^*$. It is well known that $F$ is minimized by $\p^*$, and $F(\p^*)=0$.

Finally, given a vector field $v: \Re^d \to \Re^d$ and a distribution $\pmu\in \Pspace(\Re^d)$, we define the $\Lp{\pmu}$-norm of $v$ as 
$$\|v\|_{\Lp{\pmu}}\triangleq \sqrt{\mathbb{E}_{\pmu}[\|v(x)\|_2^2]}$$

\subsection{Background on Wasserstein distance and curves in $\Pspace(\Re^d)$}
\label{s:background}
Given two distributions $\pmu,\pnu\in \Pspace(\Re^d)$, let $\Gamma(\pmu,\pnu)$ be the set of all joint distributions over the product space $\Rd \times \Rd$ whose marginals equal $\pmu$ and $\pnu$ respectively. ($\Gamma$ is the set of all couplings)

The \textbf{Wasserstein distance} is defined as
$$W_2(\pmu,\pnu) = \sqrt{\inf_{\gamma\in \Gamma(\pmu,\pnu)} \int (\|x-y\|_2^2) d\gamma(x,y)}$$

Let $(X_1, \mathscr{B}(X_1))$ and $(X_2, \mathscr{B}(X_2))$ be two measurable spaces, $\mathbf{\pmu}$ be a measure, and $r: X_1 \to X_2$ be a measurable map. The \textbf{push-forward measure} of $\mathbf{\pmu}$ through $r$ is defined as $$r_\#\mathbf{\pmu}(B) = \mathbf{\pmu}(r^{-1} (B)) \quad \forall B\in \mathscr{B}(X_2)$$
Intuitively, for any $f$, $\mathbb{E}_{r_\# \pmu}[f(x)] = \mathbb{E}_{\pmu}[f(r(x))]$.

It is a well known result that for any two distributions $\pmu$ and $\pnu$ which have density wrt the Lebesgue measure, the optimal coupling is induced by a map $T_{opt} : \Rd \to \Rd$, i.e. $W_2^2(\pmu,\pnu) = \int (\|x-y\|_2^2) d\gamma^*(x,y)$ for 
$$\gamma^* = (Id, T_{opt})_\# \pmu$$
Where $Id$ is the identity map, and $T_{opt}$ satisfies $T_{opt\#}\pmu=\pnu$, so by definition, $\gamma^*\in \Gamma(\pmu,\pnu)$. We call $T_{opt}$ \textbf{the optimal transport map}, and $T_{opt}-Id$ the \textbf{optimal displacement map}.

Given two points $\pnu$ and $\ppi$ in $\Pspace(\Re^d)$, a curve $\pmu_t:[0,1]\to\Pspace(\Re^d)$ is a \textbf{constant-speed-geodesic} between $\pnu$ and $\ppi$ if $\pmu_0 = \pnu$, $\pmu_1=\ppi$ and $W_2(\pmu_s,\pmu_t) = (t-s) W_2(\pnu,\ppi)$ for all $0\leq s\leq t\leq 1$. If $v_\pnu^{\ppi}$ is the optimal displacement map between $\pnu$ and $\ppi$, then the constant-speed-geodesic $\pmu_t$ is nicely characterized by
\begin{equation}\label{e:constantspeedgeodesiccharacterization}
\pmu_t = (Id + tv_\pnu^{\ppi})_\# \pnu
\end{equation}

Given a curve $\pmu_t: \Re^+ \to \Pspace(\Re^d)$, we define its \textbf{metric derivative} as 
\begin{equation}
\label{d:metricderivative}
|\pmu_t^\prime|\triangleq \lim\sup_{s\to t} \frac{W_2(\pmu_s, \pmu_t)}{|s-t|}
\end{equation}. Intuitively, this is the speed of the curve in 2-Wasserstein distance. We say that a curve $\pmu_t$ is \textbf{absolutely continuous} if $\int_a^b|\pmu_t'|^2< \infty$ for all $a,b\in \Re$. 

Given a curve $\pmu_t: \Re^+ \to \Pspace(\Re^d)$ and a sequence of velocity fields $v_t: \Re^+ \to (\Re^d \to \Re^d)$, we say that $\pmu_t$ and $v_t$ satisfy the \textbf{continuity equation} at $t$ if
\begin{equation}\label{e:continuityequation}
\ddt \pmu_t(x) + \div(\pmu_t(x)\cdot v_t(x)) = 0
\end{equation}
(We assume that $\pmu_t$ has density wrt Lebesgue measure for all $t$)
\begin{remark}
If $\pmu_t$ is a constant-speed-geodesic between $\pnu$ and $\ppi$, then $v_{\pnu}^{\ppi}$ satisfies \eqref{e:continuityequation} at $t=0$, by the characterization in \eqref{e:constantspeedgeodesiccharacterization}.
\end{remark}

We say that $v_t$ is tangent to $\pmu_t$ at $t$ if the continuity equation holds and $\|v_t + w\|_{\Lp{\pmu_t}}\leq \|v_t\|_{\Lp{\pmu_t}}$ for all $w$ such that $\div(\pmu_t\cdot w)=0$. Intuitively, $v_t$ is tangent to $\pmu_t$ if it minimizes $\|v_t\|_{\Lp{\pmu_t}}$ among all velocity fields $v$ that satisfy the continuity equation.

\section{Preliminary Lemmas}
This section presents some basic results needed for our main theorem.
\subsection{Calculus over $\Pspace(\Re^d)$}
\label{s:pspacegradientflow}
In this section, we present some crucial Lemmas which allow us to study the evolution of $F(\pmu_t)$ along a curve $\pmu_t: \Re^+ \to \Pspace(\Re^d)$. These results are all immediate consequences of results proven in \cite{Ambrosio}.

\begin{Lemma}\label{l:firstvariationislinear}
For any $\pmu\in \Pspace(\Re^d)$, let $\frac{\delta F}{\delta \pmu}(\pmu): \Re^d \to \Re$ be the first variation of $F$ at $\pmu$ defined as $\lrp{\frac{\delta F}{\delta \pmu}(\pmu)} (x) \triangleq \log\lrp{\frac{\pmu(x)}{\p^*(x)}} + 1$. Let the subdifferential of $F$ at $\pmu$ be given by 
$$w_{\pmu} \triangleq \nabla \lrp{\frac{\delta F}{\delta \pmu}(\pmu)} : \Re^d \to \Re^d$$. 
For any curve $\pmu_t : \Re^+ \to \Pspace(\Re^d)$, and for any $v_t$ that satisfies the continuity equation for $\pmu_t$ (see equation \eqref{e:continuityequation}), the following holds: 
$$\ddt F(\pmu_t) = \mathbb{E}_{\pmu_t} \left[\lin{w_{\pmu_t}(x), v_t(x)}\right]$$
\end{Lemma}

Based on Lemma \ref{l:firstvariationislinear}, we define (for any $\pmu\in \Pspace(\Re^d)$) the operator
\begin{equation}
\label{d:derivativeinv}
\D_{\pmu}(v) \triangleq \mathbb{E}_{\pmu} \left[\lin{w_{\pmu}(x), v(x)}\right] : (\Rd \to \Rd
)\to \Re
\end{equation}
\textbf{$\D_{\pmu}(v)$ is linear in $v$}.

\begin{Lemma}\label{l:metricderivativeandvnorm}
Let $\pmu_t$ be an absolutely continuous curve in $\Pspace(\Re^d)$ with tangent velocity field $v_t$. Let $|\pmu_t^\prime|$ be the metric derivative of $\pmu_t$.

Then $$\|v_t\|_{\Lp{\pmu_t}}=  |\pmu_t^\prime|$$
\end{Lemma}

\begin{Lemma}\label{l:firstvariationdualnorm}
For any $\pmu \in \Pspace(\Re^d)$, let $\|\D_{\pmu}\|_*\triangleq \sup_{\|v\|_{\Lp{\pmu}}\leq 1}\D_{\pmu}(v)$, then 
$$\|\D_{\pmu}\|_* = \sqrt{\int \lrn{\nabla \lrp{\frac{\delta F}{\delta \pmu} (\pmu)}(x)}_2^2 \pmu(x) dx}$$

Furthermore, for any absolutely continuous curve $\pmu_t : \Re^+ \to \Pspace(\Re^d)$ with tangent velocity $v_t$, we have
$$\left|\ddt F(\pmu_t)\right| \leq \|\D_{\pmu_t}\|_* \|v_t\|_{\Lp{\pmu_t}}$$
\end{Lemma}

As a Corollary of Lemma \ref{l:metricderivativeandvnorm} and Lemma \ref{l:firstvariationdualnorm}, we have the following result:
\begin{corollary}\label{c:upperboundratebymetricderivative}
Let $\pmu_t$ be an absolutely continuous curve with tangent velocity field $v_t$. Then $$\ddt F(\pmu_t) \leq \|\D_{\pmu_t}\|_* \cdot |\pmu_t^\prime|$$
\end{corollary}

\subsection{Exact and Discrete Gradient Flow for $F(\p)$}
In this section, we will study the curve $\p_t : \Re^+ \to \Pspace(\Re^d)$ defined in \eqref{e:discretelangevin}. Unless otherwise specified, we will assume that $\p_0$ is an arbitrary distribution.

Let $x_t$ be as defined in \eqref{e:discretelangevin}.

For any given $t$ and for all $s$, we define a stochastic process $y_s^t$ as 
\begin{align*}
y^t_s &= x_s & \text{for}\ s\leq t \\
d y^t_s &= -\nabla U(y_s^t) ds + \sqrt{2} dB_s & \text{for}\ s\geq t
\numberthis \label{e:ystdynamic}\\
&\text{let $\q_s^t$ denote the distribution for $y_s^t$}
\end{align*}
From $s=t$ onwards, this is the exact Langevin diffusion with $\p_t$ as the initial distribution (compare with expression \eqref{e:exactlangevin}).

Finally, for each $t$, we define a sequence $z_s^t$ by
\begin{align*}
z^t_s &= x_s & \text{for}\ s\leq t \\
\numberthis \label{e:zstdynamic}
d z^t_s &= (-\nabla U(z^t_{\tau(t)}) + \nabla U(z_s^t)) ds, & \text{for}\ s\geq t\\
&\text{let $\g_s^t$ denote the distribution for $z_s^t$}
\end{align*}
$z^t_s$ represents the discretization error of $\p_s$ through the divergence between $\q_s^t$ and $\p_s$ (formally stated in Lemma \ref{l:addingvelocity}). Note that $z^t_{\tau(t)} = x^t_{\tau(t)}$ because $\tau(t) \leq t$.

\begin{remark}
The the $B_s$ in \eqref{e:discretelangevin}, \eqref{e:ystdynamic} and \eqref{e:zstdynamic}) are the same. Thus, $x_s$ (from \eqref{e:discretelangevin}), $y_s^t$ (from \eqref{e:ystdynamic}) and $z_s^t$ (from \eqref{e:zstdynamic}) define a coupling between the the curves $\p_s$, $\q_s^t$ and $\g_s^t$.
\end{remark}

Our proof strategy is as follows: 

\begin{enumerate}
\item In Lemma \ref{l:addingvelocity}, we demonstrate that the divergence between $\p_s$ (discretized Langevin) and $\q_s^t$ (exact Langevin) can be represented as a curve $\g_s^t$. 
\item In Lemma \ref{l:exactflowisfast}, we demonstrate that the "decrease in $F(\p_t)$ due to exact Langevin" given by $\at{\dds F(\q_s^t)}{s=t}$ is sufficiently negative. 
\item In Lemma \ref{l:metricderivativefordiscretizationissmall}, we show that the "discretization error" given by $\at{\dds (F(\p_s) - F(\q_s^t))}{s=t}$ is small.
\item Added together, they imply that $\at{\dds F(\p_s)}{s=t}$ is sufficiently negative.
\end{enumerate}

\begin{Lemma}\label{l:addingvelocity}
For all $x\in \Re^d$ and $t\in\Re^+$
$$\at{\dds \g_s^t(x)}{s=t} = \at{(\dds \p_s(x) - \dds \q_s^t(x))}{s=t}$$
\end{Lemma}

\begin{Lemma}\label{l:exactflowisfast}
For all $s,t\in \Re^+$
$$\dds F(\q_s^t) = - \|\D_{\q_s^t}\|_*^2$$
\end{Lemma}

\begin{Lemma}\label{l:metricderivativefordiscretizationissmall}
For all $t\in \Re^+$
\begin{align*}
\at{\dds \left(F(\p_s) - F(\q_s^t)\right)}{s=t}\leq &\left(2L^2 h \sqrt{\Ep{\p_{\tau(t)}}{\|x\|_2^2}} + 2L \sqrt{hd}\right) \cdot \|\D_{\p_t}\|_*
\end{align*} 
\end{Lemma}

\section{Strong Convexity Result}
In this section, we study the consequence of assuming $m$ strong convexity and $L$ smoothness of $U(x)$.
\subsection{Theorem statement and discussion}

\begin{theorem}\label{t:strongconvexityconvergence}

Let $x_t$ and $\p_t$ be as defined in \eqref{e:discretelangevin} with $\p_0 = N(0,\frac{1}{m})$.

If 
$$h=\frac{m\epsilon}{16dL^2}$$
and
$$k= 16\frac{L^2}{m^2} \frac{d\log \frac{dL}{m\epsilon}}{\epsilon}$$

Then
$\KL{\p_{kh}}{\p^*}\leq \epsilon$
\end{theorem}

The above theorem immediately allows us to obtain the convergence rate of $\p_{kh}$ in both total variation and 2-Wasserstein distance.
\begin{corollary}\label{c:convergenceintvandw2}
Using the choice of $k$ and $h$ in Theorem \ref{t:strongconvexityconvergence}, we get
\begin{enumerate}
\item $d_{TV}(\p_{kh}, \p^*)\leq \sqrt{\epsilon}$
\item $W_2(\p_{kh}, \p^*)\leq \sqrt{\frac{2\epsilon}{m}}$
\end{enumerate}
\end{corollary}

The first item follows from Pinsker's inequality. The second item follows from \eqref{e:1storderstrongconvexity}, where we take $\pmu_0$ to be $\p^*$ and $\pmu_1$ to be $\p_{kh}$, and noting that $\D_{\p^*}=0$. To achieve $\delta$ accuracy in Total Variation or $W_2$, we apply Theorem \ref{t:strongconvexityconvergence} with $\epsilon = \delta^2$ and $\epsilon = m\delta^2$ respectively.

\begin{remark}
The $\log $ term in Theorem \ref{t:strongconvexityconvergence} is not crucial. One can run \eqref{e:langevinalgorithm} a few times, each time aiming to only halve the objective $F(\p_t) - F(\p^*)$ (thus the stepsize starts out large and is also halved each subsequent run). The proof is quite simple and will be omitted.
\end{remark}

\subsection{Proof of Theorem \ref{t:strongconvexityconvergence}}
We now state the Lemmas needed to prove Theorem \ref{t:strongconvexityconvergence}. We first establish a notion of strong convexity of $F(\pmu)$ with respect to $W_2$ metric.

\begin{Lemma}\label{l:klisstronglyconvex}
If $\log \p^*(x)$ is $m$ strongly convex, then 
\begin{equation}\label{e:0thorderstrongconvexity}
F(\pmu_t)\leq (1-t)F(\pmu_0) + t F(\pmu_1) - \frac{m}{2}t(1-t) W_2^2(\pmu_0,\pmu_1) 
\end{equation}
for all $\pmu_0, \pmu_1\in \Pspace(\Re^d)$ and $t\in[0,1]$, let $\pmu_t:[0,1]\to \Pspace(\Re^d)$ be the constant-speed geodesic between $\pmu_0$ and $\pmu_1$. (recall from \eqref{e:constantspeedgeodesiccharacterization} that If $v_{\pmu_0}^{\pmu_1}$ is the optimal displacement map from $\pmu_0$ to $\pmu_1$, then $\pmu_t = (Id+t\cdot v_{\pmu_0}^{\pmu_1})_\# \pmu_0$.)

Equivalently,
\begin{equation}\label{e:1storderstrongconvexity}
F(\pmu_1) \geq F(\pmu_0) + \D_{\pmu_0}(v_{\pmu_0}^{\pmu_1}) + \frac{m}{2}W_2^2(\pmu_0, \pmu_1)
\end{equation}
We call this the m-strong-geodesic-convexity of $F$ wrt the $W_2$ distance.
\end{Lemma}

Next, we use the $m$ strong geodesic convexity of $F$ to upper bound  $F(\pmu) - F(\p^*)$ by $\frac{1}{2m}\|\D_{\pmu}\|_*^2$ (for any $\pmu\in \Pspace(\Re^d)$). This is analogous to how $f(x) - f(x^*) \leq \frac{1}{2m}\|\nabla f(x)\|_2^2$ for standard $m$-strongly-convex functions in $R^d$.
\begin{Lemma}\label{l:upperboundbygradientdualnorm}
Under our assumption that $-\log \p^*(x)$ is $m$ strongly convex, we have that for all $\pmu\in \Pspace(\Re^d)$,
$$F(\pmu) - F(\p^*) \leq \frac{1}{2m}\|\D_{\pmu}\|_*^2$$
\end{Lemma}

Now, recall $\p_t$ from \eqref{e:discretelangevin}.
We use strong convexity to obtain a bound on $\Ep{\p_t}{\|x\|_2^2}$ for all $t$. This will be important for bounding the discretization error in conjunction with Lemma \ref{l:metricderivativefordiscretizationissmall}

\begin{Lemma}\label{l:strongconvexityboundexsq}
Let $\p_t$ be as defined in \eqref{e:discretelangevin}. If $\p_0$ is such that $\Ep{\p_0}{\|x\|_2^2}\leq \frac{4d}{m}$, and $h\leq \frac{1}{L}$ in the definition of \eqref{e:discretelangevin}, then  for all $t\in \Re^+$,
$$E_{\p_t} \|x\|^2\leq \frac{4d}{m}$$
\end{Lemma}
Finally, we put everything together to prove Theorem \ref{t:strongconvexityconvergence}.

\begin{proof}{Proof of Theorem \ref{t:strongconvexityconvergence}}

We first note that $h=\frac{m\epsilon}{16L^2}\leq \frac{1}{L}$. 

By Lemma \ref{l:strongconvexityboundexsq}, for all $t$, $\Ep{\p_t}{\|x\|_2^2}\leq \frac{4d}{m}$. Combined with Lemma \ref{l:metricderivativefordiscretizationissmall}, we get that for all $t\in \Re^+$
$$\at{\dds F(\p_s) - F(\q_s^t)}{s=t}\leq \left(4L^2h \sqrt{\frac{d}{m}} + 2L \sqrt{hd}\right) \cdot \|\D_{\p_t}\|_*$$

Suppose that $F(\p_t) - F(\p^*)\geq \epsilon$, and let
$$h= \frac{m\epsilon}{16dL^2} \leq \frac{1}{16}\min\lrbb{ \frac{m}{L^2}\sqrt{\frac{\epsilon}{d}},\frac{m\epsilon}{L^2 d}}$$
then $\forall t$
\begin{align*}
 \at{\dds F(\p_s) - F(\q_s^t)}{s=t} \leq & \left(4L^2 h \sqrt{\frac{d}{m}} + 2L \sqrt{hd}\right)\\
\leq &  \frac{1}{2}\sqrt{m\epsilon}\|\D_{\p_t}\|_* \leq \frac{1}{2}\|\D_{\p_t}\|_*^2
\end{align*}
Where the last inequality is because Lemma \ref{l:upperboundbygradientdualnorm} and the assumption that $F(\p_t) - F(\p^*) \geq \epsilon$ together imply that $\|\D_{\p_t}\|_*\geq \sqrt{2m\epsilon}$. 

So combining Lemma \ref{l:exactflowisfast} and Lemma \ref{l:addingvelocity}, we have
\begin{align*}
\ddt F(\p_t) 
&= \at{\dds F(\q_s^t)}{s=t}+ \at{\dds F(\p_s) - F(\q_s^t)}{s=t}\\
&\leq -\|\D_{\p_t}\|_*^2 + \frac{1}{2} \|\D_{\p_t}\|_*^2\\
&= -\frac{1}{2} \|\D_{\p_t}\|_*^2\\
&\leq -m(F(\p_t) - F(\p^*))
\numberthis \label{e:ptdecay}
\end{align*}
Where the last line once again follows from Lemma \ref{l:upperboundbygradientdualnorm}.

To handle the case when $F(\p_t) - F(\p^*) \leq \epsilon$, we use the following argument:
\begin{enumerate}
\item We can conclude that $F(\p_t)- F(\p^*)\geq \epsilon$ implies $\ddt F(\p_t) \leq 0$. 

\item By the results of Lemma \ref{l:regularityofdensitiesandcurves} and Lemma \ref{l:pthasboundedmetricderivative}, for all $t$, $|\p_t'|$ is finite and $\|\D_{\p_t}\|$ is finite, so $\ddt F(\p_t)$ is finite and $F(\p_t)$ is continuous in $t$.

\item Thus, if $F(\p_t)\leq \epsilon$ for some $t\leq kh$, then $F(\p_s)\leq \epsilon$ for all $s\geq t$ as $F(\p_t)\geq \epsilon$ implies $\ddt F(\p_t) \leq 0$ and $F(\p_t)$ is continuous in $t$. Thus $F(\p_{kh}) - F(\p^*)\leq \epsilon$.
\end{enumerate}

Thus, we need only consider the case that $F(\p_t) \geq \epsilon$ for all $t\leq kh$. This means that \eqref{e:ptdecay} holds for all $t\leq kh$.

By Gronwall's inequality, we get 
$$F(\p_{kh}) - F(\p^*)\leq \lrp{F(\p_0) - F(\p^*)} \exp(-mkh)$$

We thus need to pick
$$k = \frac{\frac{1}{m}\log\frac{F(\p_0)-F(\p^*)}{\epsilon}}{h} = 16\frac{L^2}{m^2}\frac{d\log\frac{F(\p_0)-F(\p^*)}{\epsilon}}{\epsilon}$$

Using the fact that $\p_0=N(0,\frac{1}{m})$. Using $L$-smoothness and $m$-strong convexity, we can show that
$$-\log \p^*(x) \leq \frac{L}{2}\|x\|_2^2 + \frac{d}{2}\log(\frac{2\pi}{m})$$, 
and
$$\log {\p_0}(x) = -\frac{m}{2}\|x\|_2^2 - \frac{d}{2}\log(\frac{2\pi}{m})$$
. We thus get that $F(\p_0) - F(\p^*) = \KL{\p_0}{\p^*}\leq \frac{dL}{m}$, so 
$$k = 16\frac{L^2}{m^2} \frac{d\log \frac{dL}{m\epsilon}}{\epsilon}$$
\end{proof}

\section{Weak convexity result}
\label{s:nonsc}
In this section, we study the case when $\log \p^*$ is not $m$ strongly convex (but still convex and $L$ smooth). Let $\ppi_h$ be the stationary distribution of \eqref{e:discretelangevin} with stepsize $h$.

We will assume that we can choose an initial distribution $\p_0$ which satisfies 
\begin{equation}\label{e:c1definition}
W_2(\p_0, \p^*)= C_1
\end{equation} and  
\begin{equation}\label{e:c2definition}
\sqrt{E_{\p^*}\|x\|_2^2} = C_2
\end{equation}.
Let $h^\prime$ be the largest stepsize such that 
\begin{equation}\label{e:hprimeassumption}
W_2(\ppi_h, p^*)\leq C_1 \quad, \forall h\leq h^\prime
\end{equation}

\subsection{Theorem statement and discussion}
\begin{theorem}\label{t:nonstrongconvexityconvergence}
Let $C_1$, $C_2$ and $h^\prime$ be defined as in the beginning of this section. 

Let $x_t$ and $\p_t$ be as defined in \eqref{e:discretelangevin} with $\p_0$ satisfying \eqref{e:c1definition}.
If 
\begin{align*}
h
&=\frac{1}{48}\min\left\{\frac{\epsilon}{C_1(C_1+C_2)L^2}, \frac{\epsilon^2}{ C_1^2dL^2}, h^\prime\right\}\\
&=\frac{1}{48}\min\left\{\frac{\epsilon}{C_1 C_2 L^2}, \frac{\epsilon^2}{C_1^2d L^2}, h^\prime\right\}
\end{align*}
and
$$k = \frac{2C_1^2}{\epsilon h} + \frac{2C_1^2 \log (F(\r^0) - F(\p^*))}{h}$$

Then
$\KL{\p_{kh}}{\p^*}\leq \epsilon$

\end{theorem}

Once again, applying Pinsker's inequality, we get that the above choice of $k$ and $t$ yields $d_{TV}(\r^k,\p^*) \leq \sqrt{\epsilon}$. Without strong convexity, we cannot get a bound on $W_2$ from bounding $F(\r^k)-F(\p^*)$ like we did in corollary \ref{c:convergenceintvandw2}.

In \cite{theoretical}, a proof in the non-strongly-convex case was obtained by running Langevin MCMC on 
$$ \tilde{\p}^* \propto  \p^* \cdot \exp(-\frac{\delta}{d} \|x\|_2^2) $$
$\log \tilde{\p}^*$ is thus strongly convex with $m = \frac{\delta}{d}$, and $d_{TV}(\p^*, \tilde{\p}^*)\leq \delta$. By the results of \cite{theoretical}, or \cite{nonasymptotic}, or Theorem \ref{t:strongconvexityconvergence}, we need
\begin{equation}\label{e:nonscdalalyan}
k = \tilde{O}(\frac{d^3}{\delta^4})
\end{equation}
iterations to get $d_{TV}(\p_{kh},\p^*)\leq \delta$.

On the other hand, if we assume $\log (F(\p_0) - F(\p^*)) \leq \frac{1}{\epsilon}$ and $ h^\prime\geq \frac{1}{10}\min\left\{\frac{\epsilon}{C_1 C_2 L^2}, \frac{\epsilon^2}{C_1^2d L^2}\right\}$
the results of Theorem \ref{t:nonstrongconvexityconvergence} implies that 
$$h = \frac{\epsilon}{L^2 C_1} \min\left\{\frac{1}{C_2}, \frac{\epsilon}{dC_1}\right\}$$
To get $d_{TV}(\p_{kh},\p^*)\leq \delta$, we need
$$k = \frac{L^2 C_1^3}{\delta^4} \max\left\{C_2, \frac{dC_1}{\delta^2}\right\}$$
Even if we ignore $C_1$ and $C_2$, our result is not strictly better than \eqref{e:nonscdalalyan} as we have a worse dependence on $\delta$. However, we do have a better dependence on $d$.

The proof of Theorem \ref{t:nonstrongconvexityconvergence} is quite similar to that of Theorem \ref{t:strongconvexityconvergence}, so we defer it to the appendix.

\newpage
\null
\newpage
\section{Supplementary Materials}

\begin{proof}{Proof of Lemma \ref{l:firstvariationislinear}}
The proof is directly from results in \cite{Ambrosio}. See Theorem 10.4.9, with $\mathcal{F}(\mu|\gamma) = \KL{\mu}{\gamma}$, with $\mu=\pmu$, $\gamma = \p^*$, $\sigma = \frac{\pmu}{\p^*}$, $F(\rho) = \rho \log \rho$, $L_F(\sigma) = \sigma$, and $w_{\pmu} = \frac{\nabla L_F(\sigma)}{\sigma} = \nabla \log \frac{\pmu}{\p^*}$. The expression for $\ddt F(\p_t)$ comes from expression 10.1.16 (section E of chapter 10.1.2, page 233). See also expressions 10.4.67 and 10.4.68.

(One can also refer to Theorem 10.4.13 and Theorem 10.4.17 for proofs of $w_{\pmu}$ for the KL-divergence functional in more general settings.) By Lemma \ref{l:regularityofdensitiesandcurves}, $w_{\p_t}$ is well defined for all $t$. 
\end{proof}

\begin{proof}{Proof of Lemma \ref{l:metricderivativeandvnorm}}
Theorem 8.3.1 of \cite{Ambrosio}.

\end{proof}

\begin{proof}{Proof of Lemma \ref{l:firstvariationdualnorm}}
By definition of $\D_{\pmu_t}(v)$ in \eqref{d:derivativeinv} and Lemma \ref{l:firstvariationislinear} and Cauchy Schwarz. 
\end{proof}

\begin{proof}{Proof of Lemma \ref{l:addingvelocity}}
In this proof, we treat $t$ as a fixed but arbitrary number, and prove the Lemma for all $t\in \Re^+$. We will use $x_s$, $y_s^t$, $z_s^t$, $\p_s$, $\q_s^t$ and $\g_s^t$ as defined in \eqref{e:discretelangevin}, \eqref{e:ystdynamic} and \eqref{e:zstdynamic}.

First, consider the case when $t=\tau(t)$. By definition, $x_t=y_t^t=z_t^t$, and $\p_t = \q_t^t = \g_t^t$. By Fokker Planck,
\begin{align*}
\at{\dds \p_s(x)}{s=t} 
&= -\nabla U(x_t) + \tr(\nabla^2 \p_t)\\
&= -\nabla U(y_t^t) + \tr(\nabla^2 \q_t^t)\\
&= \at{\dds \q_s^t(x)}{s=t}
\end{align*}
On the other hand
$$ \at{d z_s^t}{s=t} = -\nabla U(z_{\tau(t)}^t) + \nabla U(z_t^t) = -\nabla U(x_t) + \nabla U(x_t)=0$$
Thus
$\at{\dds \g_s^t}{s=t}=0$
So Lemma \eqref{l:addingvelocity} holds.

In the remainder of this proof, we assume that $t\neq \tau(t)$.

For a given $\Theta\in \Re^{2d}$, we let $\Pi_1(\Theta)$ denote the projection of $\Theta$ onto its first $d$ coordinates, and $\Pi_2(\Theta)$ denote the projection of $\Theta$ onto its last $d$ coordinates. With abuse of notation, for $\P\in \Pspace(\Re^{2d})$, we let $\Pi_1 (\P)$ and $\Pi_2 (\P)$ denote the corresponding marginal densities.

We will consider three stochastic processes: $\Theta_s, \Lambda_s^t, \Psi_s^t$ over $\Re^{2d}$ for $s\in [\tau(t), \tau(t) + h)$.

First, we introduce the stochastic process $\Theta_s$ for $s\in [\tau(t), \tau(t) + h)$
\begin{align*}
\Theta_{\tau(t)} &= \cvec{x_{\tau(t)}}{-\nabla U(x_{\tau(t)})}\\
d \Theta_s &= \cvec{\Pi_2(\Theta_s)}{0} dt + \cvec{\sqrt{2}dB_t}{0} \quad  \text{for } s\in [\tau(t), \tau(t)+h)
\end{align*}
We let $\P_s$ denote the density for $\Theta_s$. Intuitively, $\P_s$ is the joint density between $x_s$ and $-\nabla U(x_{\tau(t)})$. One can verify that $\Pi_1(\Theta_s) = x_s$ and $\Pi_1(\P_s)=\p_s$. By Fokker-Planck, we have $\forall \Theta\in \Re^{2d}$
\begin{align*}
\at{\dds \P_s (\Theta)}{s=t} = &-\nabla \cdot \lrp{\P_t (\Theta) \cdot  \cvec{\Pi_2(\Theta)}{0}} \\
&\quad + \sum_{i=1}^d \frac{\del^2}{\del \Theta_i^2} \P_t(\Theta)
\numberthis \label{e:Pfokkerplanck}
\end{align*}

Next, for any given $t$, we introduce the stochastic process $\Lambda_s^t$ for $s\in [\tau(t), \tau(t) + h)$.
\begin{align*}
\Lambda_s^t &= \Theta_s &\text{for } s\leq t \\
d \Lambda_s^t &= \cvec{-\nabla U(\Pi_1 (\Lambda_s^t))}{0}ds + \cvec{\sqrt{2} dB_s}{0} & \text{for}\ s\geq t
\end{align*}
Let $\Q_s^t$ denote the density for $\Lambda_s^t$. One can verify that $\Pi_1(\Lambda_s^t) = y_s^t$ and $\Pi_1(\Q_s^t)=\q_s^t$. By Fokker-Planck, we have $\forall \Theta\in \Re^{2d}$
\begin{align*}
\at{\dds \Q_s^t (\Theta)}{s=t} =& -\nabla \cdot \lrp{\Q_t^t (\Theta) \cdot  \cvec{-\nabla U(\Pi_1 (\Theta))}{0}} \\
&\quad + \sum_{i=1}^d \frac{\del^2}{\del \Theta_i^2} \Q_t^t(\Theta)
\numberthis \label{e:Qfokkerplanck}
\end{align*}

Finally, define
\begin{align*}
\Psi_s^t &= \Theta_s &\text{for } s\leq t \\
d \Psi_s^t &= \cvec{\Pi_2(\Psi_s^t)+\nabla U(\Pi_1 (\Psi_s^t))}{0}ds \\
&\quad + \cvec{\sqrt{2} dB_s}{0} & \text{for}\ s\geq t
\end{align*}
Let $\G_s^t$ denote the density for $\Psi_s^t$. One can verify that $\Pi_1(\Psi_s^t) = z_s^t$ and $\Pi_1(\G_s^t)=\g_s^t$. By Fokker-Planck, we have $\forall \Theta\in \Re^{2d}$

\begin{equation}
\label{e:Rfokkerplanck}
\at{\dds \G_s^t (\Theta)}{s=t} = -\nabla \cdot \lrp{\G_t^t (\Theta) \cdot  \cvec{\Pi_2(\Theta)+\nabla U(\Pi_1 (\Theta))}{0}}
\end{equation}

By definition, $\Theta_t = \Lambda_t^t = \Psi_t^t$ almost surely, and  $\P_t = \Q_t^t = \G_t^t$. Taking the difference between  \eqref{e:Pfokkerplanck}, \eqref{e:Qfokkerplanck} thus gives
\begin{align*}
\at{\dds \P_s(\Theta) - \Q_s^t(\Theta) }{s=t} =& -\nabla \cdot \lrp{\P_t \lrp{\Theta}\cdot  \cvec{\Pi_2(\Theta) + \nabla U(\Pi_1(\Theta))}{0}}\\
=& \at{\dds \G_s^t \lrp{\Theta}}{s=t}
\end{align*}
Finally, marginalizing out the last $d$ coordinates on both sides, and recalling that $\Pi_1(\P_s) = \p_s$, $\Pi_1(\Q_s^t) = \q_s^t$ and $\Pi_1(\G_s^t) = \g_s^t$, we prove the Lemma.
\end{proof}

\begin{proof}{Proof of Lemma \ref{l:exactflowisfast}}
The fact that $\q^t_s$ is the steepest descent follows from the fact that Fokker-Planck equation for Langevin diffusion yields, for all $x\in \Re^d$
\begin{align*}
\dds \q^t_s(x) = &\div (\q^t_s(x) \nabla \log \p^*(x)) +  \tr(\nabla^2 \q^t_s(x)) \\
= &\div \left(\q^t_s(x) \lrp{\nabla \log \frac{\p^*(x)}{\q^t_s(x)}}\right)
\end{align*}
By definition of \eqref{e:continuityequation}, we get that 
\begin{equation}
\label{e:tangentvelocityforqst}
v_s = \nabla \log \frac{\p^*(x)}{\q^t_s(x)}
\end{equation}
satisfies the continuity equation for $\q_s^t$. By Lemma \ref{l:firstvariationislinear},
$$w_{\q_s^t} = \nabla \log\lrp{\frac{\q_s^t}{\p^*}}$$ Thus
$$\dds F(\q_s^t) = \D_{\q_s^t}(v_s) = -\Ep{\q_s^t}{\|w_{\q_s^t}\|_2^2} = - \|\D_{\q_s^t}\|_*^2$$
Where the last equality is by Cauchy-Schwarz

\end{proof}

\begin{proof}{Proof of Lemma \ref{l:metricderivativefordiscretizationissmall}}
Consider $z_s^t$ and $\g_s^t$ as defined in \eqref{e:zstdynamic}. By Lemma \ref{l:addingvelocity}, $\at{\dds \g_s^t}{s=t} = \at{(\dds \p_s - \dds \q_s^t)}{s=t}$. The first variation of $F$, defined by 
$$\lim_{\epsilon\to 0}\frac{F(\pmu + \epsilon \Delta) - F(\pmu)}{\epsilon}=\int{\lrp{\frac{\delta F}{\delta \pmu}(\pmu)}(x) \cdot \Delta(x)} dx$$
 is linear (see Chapter 7.2 of \cite{Santambrogiobook}). (In the above, $\Delta: \Re^d \to \Re$ is an arbitrary $0$-mean perturbation). In addition, because $\p_t = \q_t^t = \g_t^t$, we have $\frac{\delta F}{\delta \pmu}(\p_t)=\frac{\delta F}{\delta \pmu}(\q_t^t) = \frac{\delta F}{\delta \pmu}(\g_t^t)$, we get that 
$$\at{\dds F(\g_s^t)}{s=t} = \at{\left(\dds F(\p_s) - \dds F(\q_s^t)\right)}{s=t}$$

We will upper bound $\at{|\g_s^{t\prime}|}{s=t}$, then apply Corollary \ref{c:upperboundratebymetricderivative}.
\begin{align*}
&\at{|\g_s^{t\prime}|}{s=t}\\
=&\lim_{\epsilon\to 0} \frac{1}{\epsilon} W_2(\g_{t+\epsilon}^t, \g_t^t) \\
\leq& \lim_{\epsilon\to 0}\frac{1}{\epsilon}\sqrt{\E{\left\| \epsilon( \nabla U(x_t) -\nabla U(x_{\tau(t)}) ) \right\|_2^2 }}\\
=& \sqrt{\E{\left\| \nabla U(x_t) -\nabla U(x_{\tau(t)}) \right\|_2^2}}\\
\leq& \sqrt{\E{L^2\|x_t-x_{\tau(t)}\|_2^2}}\\
=& L\sqrt{\E{\|(t-\tau(t))\nabla U(x_{\tau(t)}) + \sqrt{2} (B_t - B_{\tau(t)}) \|_2^2}}\\
\leq& 2L(t-\tau(t))\sqrt{\E{\|\nabla U(x_{\tau(t)})\|_2^2}} + 2L\sqrt{(t-\tau(t))d}\\
\leq& 2L(t-\tau(t))\sqrt{L^2\E{\|x_{\tau(t)}\|_2^2}} + 2L\sqrt{(t-\tau(t))d}
\end{align*}

Where the first line is by definition of metric derivative, second line is by the coupling between $\g_t^t$ and $\g_{t+\epsilon}^t)$ induced by the joint distribution $(z_t^t, z_{t+\epsilon}^t)$ and the fact that $z_{\tau(t)}^t = x_t$. The fourth line is by Lipschitz-gradient of $U(x)$, fifth line is by definition of $x_t$, sixth line is by variance of $B_t - B_0$, seventh line is once again by Lipschitz-gradient of $U(x)$.

Thus, we upper bound $\at{|\g_s^{t\prime}|}{s=t}$  by $2L^2(t-\tau(t))\sqrt{\E{\|x_{\tau(t)}\|_2^2}} + 2L\sqrt{(t-\tau(t))d}$. Applying Corollary \ref{c:upperboundratebymetricderivative}, and using the fact that for all $t$, $t-\tau(t) \leq h$, we get
\begin{align*}
&\at{\dds F(\g_s^t)}{s=t}\\
\leq &\left(2L^2h\sqrt{\E{\|x_{\tau(t)}\|_2^2}} + 2L\sqrt{hd}\right) \|\D_{\g_t^t}\|_*\\
\leq &\left(2L^2h\sqrt{\E{\|x_{\tau(t)}\|_2^2}} + 2L\sqrt{hd}\right) \|\D_{\p_t}\|_*
\end{align*}
The last line is because $\g_t^t=\p_t$ by definition.

\end{proof}

\begin{proof}{Proof of Lemma \ref{l:klisstronglyconvex}}
By Theorem 9.4.11 of \cite{Ambrosio},  $m$-strong-convexity of $\log \p^*$ implies geodesic convexity. Expression \eqref{e:0thorderstrongconvexity} then follows from the definition of geodesic convexity in definition 9.1.1 of \cite{Ambrosio}.

Rearrranging terms, dividing by $t$ and taking limit as $t\to 0$, we get
\begin{align*}
F(\pmu_1) 
&\geq F(\pmu_0) + \lim_{t\to 0 }\frac{F(\pmu_t) - F(\pmu_0)}{t} + \frac{m}{2}W_2^2(\pmu_0,\pmu_1)\\
&= F(\pmu_0) + \D_{\pmu_0}(v_{\pmu_0}^{\pmu_1}) + \frac{m}{2}W_2^2(\pmu_0,\pmu_1)
\end{align*}
The last equality follows by Lemma \ref{l:firstvariationislinear} and by the remark immediately following \eqref{e:continuityequation}.

We remark that the proof of \eqref{e:1storderstrongconvexity} is completely analogous to the proof of first-order characterization of strongly convex functions over $\Re^d$.
 
\end{proof}

\begin{proof}{Proof of Lemma \ref{l:upperboundbygradientdualnorm}}
We consider \eqref{e:1storderstrongconvexity}, and use two facts
\begin{enumerate}
\item For any $\pmu\in\Pspace(\Re^d)$, $\D_{\pmu}(v)$ is linear in $v$. (see \eqref{d:derivativeinv})
\item For any $\pmu,\pnu\in\Pspace(\Re^d)$, $W_2^2(\pmu, \pnu)= \mathbb{E}_{\pmu}{\|v_{\pmu}^{\pnu}(x)\|_2^2}$, by definition of $W_2$ and $v_{\pmu}^{\pnu}$ as the optimal displacement map.
\end{enumerate}
We apply Lemma \eqref{l:upperboundbygradientdualnorm} with $\pmu_0=\p^*$ and $\pmu_1=\pmu$. Let $v_{\pmu}^{\p^*}$ be the optimal displacement map from $\pmu$ to $\p^*$, so \eqref{e:1storderstrongconvexity} gives
\begin{align*}
F(\pmu) - F(\p^*) 
&\leq - \D_{\pmu}(v_{\pmu}^{\p^*}) - \frac{m}{2}W_2^2(\pmu,\p^*)\\
&= - \D_{\pmu}(v_{\pmu}^{\p^*}) - \frac{m}{2}\mathbb{E}_{\pmu}{\|v_{\pmu}^{\p^*}(x)\|_2^2}
\end{align*}
Let $v^* \triangleq \arg\max_{\|v\|_{\Lp{\pmu}}\leq 1} -\D_{\pmu}(v)$, so $\D_{\pmu}(v^*)= -\|\D_{\pmu}\|_*$ by linearity. We know that the maximizer of 
$$\arg\max_v - \D_{\pmu}(v) - \frac{m}{2}\mathbb{E}_{\pmu}{\|v_{\pmu}^{\p^*}(x)\|_2^2}= c\cdot v^*$$
for some real number $c$. Taking derivatives wrt $c$ gives $c=\frac{1}{m}\|\D_{\pmu}\|_*$. Thus we get

$$F(\pmu) - F(\p^*) \leq \frac{m}{2}\|\D_{\pmu}\|_*^2$$

\end{proof}

\begin{proof}{Proof of Lemma \ref{l:strongconvexityboundexsq}}
We prove this by induction on $k$. First, by definition of $\p_0 = N(0,\frac{1}{m})$, we get that 
$$\Ep{\p_t}{\|x\|_2^2} =\frac{d}{m} \leq \frac{4d}{m}\quad , \forall t\leq 0h$$

Next, we assume that for some $k$, and for all $t\leq kh$, $\Ep{\p_t}{\|x\|_2^2}\leq \frac{4d}{m}$.

For the inductive step, we consider $t\in (kh, (k+1)h]$

From \eqref{e:discretelangevin}, 
\begin{align*}
x_t 
&= x_{kh} - (t-kh)\nabla U(x_{kh}) + \sqrt{2} (B_t - B_{kh})
\end{align*}
By smoothness and strong convexity and the assumption that $\arg\min_x U(x) = 0$, we get that for all $x$ and for all $t$:
$$\|(x-(t-kh)\nabla U(x))-0\|_2\leq (1-mt) \|x-0\|_2$$
(note that $h\leq \frac{1}{L}$ implies that $t-kh\leq \frac{1}{L}$.)
So for all $t$
\begin{align*}
& \mathbb{E}_{x\sim\p_t} \|x\|_2^2 \\
=& \mathbb{E}_{x\sim\p_{kh}} \|x - (t-{kh}) \nabla U(x) + \sqrt{2}(B_t- B_{kh})\|_2^2\\
=& \mathbb{E}_{x\sim\p_{kh}} \|x - (t-{kh}) \nabla U(x)\|_2^2 + \mathbb{E}\|\sqrt{2}(B_t- B_{kh})\|_2^2\\
\leq& (1-mt) \mathbb{E}_{x\sim\p_{kh}} \|x\|_2^2 + 2dt\\
=& \mathbb{E}_{x\sim\p_{kh}}\|x\|_2^2  + (2dt - mt \mathbb{E}_{x\sim\p_{kh}}\|x\|_2^2 )
\end{align*}

By inductive hypothesis, we have $\mathbb{E}_{x\sim\p_{t}}\|x\|_2^2 \leq \frac{4d}{m}$ for all $t\leq kh$

If $\mathbb{E}_{x\sim\p_{kh}}\|x\|_2^2 \geq \frac{2d}{m}$, then $\mathbb{E}_{x\sim\p_t}\|x\|_2^2 \leq \mathbb{E}_{x\sim\p_{kh}}\|x\|_2^2 \leq \frac{4d}{m}$.

If $\mathbb{E}_{x\sim\p_{kh}}\|x\|_2^2 \leq \frac{2d}{m}$, then $\mathbb{E}_{x\sim\p_{t}}\|x\|_2^2\leq \frac{2d}{m} + \frac{2d}{L}\leq \frac{4d}{m}$ (by $t-{kh}\leq \frac{1}{L}$ and by $L \geq m$).

Thus if $\p_{kh}$ is such that $E_{\p_{kh}}\|x\|_2^2\leq \frac{4d}{m}$, then it must be that $E_{\p_t} \|x\|^2\leq \frac{4d}{m}$ for all $t\in(kh, (k+1)h]$, thus proving the inductive step.
\end{proof}

\subsection{Proof of Theorem \ref{t:nonstrongconvexityconvergence}}
First, we present a Lemma for upper bounding $F(\pmu) - F(\p^*)$ for $\pmu\in \Pspace(\Re^d)$ in the absence of strong convexity. The following Lemma plays an analogous role to Lemma \ref{l:upperboundbygradientdualnorm}.
\begin{Lemma}\label{l:nonscupperboundbygradientdualnorm}
Let $F$ be convex in $W_2$, then for all $\pmu\in\Pspace(\Re^d)$,
$$F(\pmu) - F(\p^*) \leq \|\D_{\pmu}\|_*W_2(\pmu,\p^*)$$
\end{Lemma}

\begin{proof}{Proof of Lemma \ref{l:nonscupperboundbygradientdualnorm}}
Similar to the proof of Lemma \ref{l:upperboundbygradientdualnorm}, we consider \eqref{e:1storderstrongconvexity}, but with $m=0$, 
(and once again $v_{\pmu}^{\p^*}$ denotes the optimal displacement map from $\pmu$ to $\p^*$):
\begin{align*}
F(\pmu) - F(\p^*) 
&\leq -\D_{\p}(v_{\pmu}^{\p^*})\\
&\leq \|\D_{\pmu}\|_* \cdot \|v_{\pmu}^{\p^*}\|_{\Lp{\pmu}}\\
&\leq \|\D_{\pmu}\|_* \cdot W_2(\pmu,\p^*)
\end{align*}
Where first inequality is from \eqref{e:1storderstrongconvexity}, second line is by definition of $\|\D_{\pmu}\|_*$, third line is by defintion of Wasserstein distance and the fact that $v_{\pmu}^{\p^*}$ is the optimal transport map.

\end{proof}

Next, we establish that for a fixed stepsize $h$, $W_2(\p_t, \ppi_h)$ is nonincreasing, using a synchronous coupling technique taken from \cite{highdim}.

\begin{Lemma}\label{l:nonstronglyconvexWassersteindistanceiscontractive}
Let $\p_t$ be defined as in the statement of Theorem \eqref{t:nonstrongconvexityconvergence}. Let $h$ be a fixed stepsize satisfying $h\leq \min\{\frac{1}{L}, h^\prime\}$. Then for all $k$,
$$W_2(\p_{kh},\ppi_h) \leq W_2(\p_0,\ppi_h)$$
\end{Lemma}
\begin{proof}{Proof of Lemma \ref{l:nonstronglyconvexWassersteindistanceiscontractive}}

First, we demonstrate that \eqref{e:discretelangevin} is contractive in $W_2$.

We will prove this by induction.

\textbf{Base case}: trivially true.

\textbf{Inductive Hypothesis}: $W_2(\p_{kh},\ppi_h) \leq W_2(\p_0,\ppi_h)$ for some $k$.

\textbf{Inductive Step}:
Let $T$ be the optimal transport map from $\p_{kh}$ to $\ppi_h$. We will demonstrate a coupling between $\p_{(k+1)h}$ and $\ppi_h$ with cost less than $W_2(\p_{kh}, \ppi_h)$. The Lemma then follows from induction.

Since $x_{kh}\sim \p_{kh}$ (see \eqref{e:discretelangevin}), the optimal coupling between $\p_{kh}$ and $\ppi_h$ is given by the pair of random variables $(x_{kh}, T(x_{kh}))$. For $t\in[kh,(k+1)h]$, 
$$x_{(k+1)h} = x_{kh} - h \nabla U(x_{kh}) + \sqrt{2} (B_{(k+1)h}-B_{kh}))$$. Consider the coupling $\gamma$ between $\p_{kh}$ and $\ppi_h$ defined by the following pair of random variables
\begin{align*}
&\left(x_{kh} - h\nabla U(x_{kh}) + \sqrt{2}(B_{(k+1)h}-B_{kh}),\right. \\
&\quad \left.\quad T(x_{kh}) - h\nabla U(T(x_{kh})) + \sqrt{2}(B_{(k+1)h}-B_{kh})\right)
\end{align*}

(Note that $\ppi_h$ is stationary under the discrete Langevin diffusion with stepsize $h$, so $\gamma$ does have the right marginals).

To demonstrate contraction in $W_2$:
\begin{align*}
&W_2^2(\p_{(k+1)h},\ppi_h)\\
\leq&\mathbb{E} \left[\left\| \left(x_{kh} - h\nabla U(x_{kh}) + \sqrt{2}(B_{(k+1)h}-B_{kh})\right) \right.\right.\\
&-\left. \left. \left(T(x_{kh}) - h\nabla U(T(x_{kh})) + \sqrt{2}(B_{(k+1)h}-B_{kh})\right)\right\|_2^2\right]\\
=& \E{\lrn{\left(x_{kh} - h\nabla U(x_{kh}) \right)- \left(T(x_{kh}) - h\nabla U(T(x_{kh}))\right)}_2^2}\\
\leq& \mathbb{E}[\|x_{kh}-T(x_{kh})\|_2^2 - 2h\lin{\nabla U(x_{kh}) - \nabla U(T(x_{kh})), x_{kh}- T(x_{kh})} \\
& + h^2\|\nabla U(x_{kh}) - \nabla U(T(x_{kh}))\|_2^2]\\
\leq& \E{\|x_{kh}-T(x_{kh})\|_2^2 }\\
=& W_2^2 (\p_{kh},\ppi_h)
\end{align*}
where the last equality follows by optimality of $T$, and the last inequality follows because $L$-smoothness of $U(x)$ implies
\begin{align*}
&-2h\lin{\nabla U(x_{kh}) - \nabla U(T(x_{kh})), x_{kh}-T(x_{kh})}\\
\leq & -\frac{h}{L}\|\nabla U(x_{kh}) - \nabla U(T(x_{kh}))\|_2^2\\
\leq & - h^2 \|\nabla U(x_{kh}) - \nabla U(T(x_{kh}))\|_2^2
\end{align*}
This completes the inductive step. 
\end{proof}

\begin{corollary}\label{c:nonstronglyconvexWassersteindistancebound}
Let $\p_t$ be as defined in \eqref{t:nonstrongconvexityconvergence}. Then for all $t$,
$$W_2(\p_t, \p^*) \leq 4C_1$$
\end{corollary}

\begin{proof}{Proof of Corollary \ref{c:nonstronglyconvexWassersteindistancebound}}
First, if $t=\tau(t)$, then by Lemma \ref{l:nonstronglyconvexWassersteindistanceiscontractive} and \eqref{e:hprimeassumption} and triangle inequality, we get our conclusion.

So assume that $t\neq \tau(t)$. Using identical arguments as in Lemma \ref{l:nonstronglyconvexWassersteindistanceiscontractive}, and noting the assumption on $h'$ in \eqref{e:hprimeassumption} and the fact that $h\leq h'$, we can show that

\begin{equation}\label{e:w2ofpttotminustaut}
W_2(\p_{t}, \ppi_{t-\tau(t)})\leq W_2(\p_{\tau(t)}, \ppi_{t-\tau(t)})
\end{equation}

By triangle inequality and the assumption in \eqref{e:hprimeassumption}, we have
\begin{align*}
& W_2(\p_t,\p^*)\\
\leq & W_2(\p_t,\ppi_{t-\tau(t)}) + W_2(\ppi_{t-\tau(t)},\p^*)\\
\leq & W_2(\p_{\tau(t)}, \ppi_{t-\tau(t)}) + W_2(\ppi_{t-\tau(t)},\p^*)\\
\leq & W_2(\p_{\tau(t)}, \ppi_h) + W_2(\ppi_h,\p^*) \\
\quad &+ W_2(\ppi_h,\p^*) + W_2(\ppi_{t-\tau(t)},\p^*)\\
\leq & 4C_1
\end{align*}
Where the first inequality is by triangle inequality, the second inequality is by \eqref{e:w2ofpttotminustaut}, third inequality is by triangle inequality, fourth inequality is by assumption \eqref{e:hprimeassumption} and the fact that $t-\tau(t) \leq h \leq h'$.

\end{proof}

Next, we use Lemma \ref{l:nonstronglyconvexWassersteindistanceiscontractive}, to bound $\E {\|x_{kh}\|_2^2}$ for all $k$:
\begin{Lemma}\label{l:vkvariancebound}
Let $h$, $x_t$ and $\p_t$ be as defined in the statement of Theorem \ref{t:nonstrongconvexityconvergence}. Then for all $k$
\begin{align*}
\E{ \|x_{kh}\|_2^2}&\leq 4(C_1^2 + C_2^2)
\end{align*}
\end{Lemma}

\begin{proof}{Proof of Lemma \ref{l:vkvariancebound}}

Let $\gamma(x,y)$ be the optimal coupling between $\p_{kh}$ and $\ppi_h$. Let $\gamma'(x,y)$ be the optimal coupling between $\ppi_h$ and $\p^*$. Then

\begin{align*}
\Ep{\p_{kh}}{\|x\|_2^2}
&= \Ep{\gamma}{\|x\|_2^2}\\
&= \Ep{\gamma}{\|x-y + y\|_2^2}\\
&\leq 2\Ep{\gamma}{\|x-y\|_2^2}+2\Ep{\gamma}{\|y\|_2^2}\\
&= 2W_2(\p_{kh},\ppi_h)+2\Ep{\ppi_h}{\|y\|_2^2}\\
&= 2W_2(\p_{kh},\ppi_h)+2\Ep{\gamma'}{\|x\|_2^2}\\
&= 2W_2(\p_{kh},\ppi_h)+2\Ep{\gamma'}{\|x-y + y\|_2^2}\\
&\leq 2W_2(\p_{kh},\ppi_h)+4\Ep{\gamma'}{\|x-y\|_2^2}+4\Ep{\gamma'}{\|y\|_2^2}\\
&\leq 2W_2(\p_{kh},\ppi_h)+4W_2(\ppi_h,\p^*)+4\Ep{\p^*}{\|x\|_2^2}
\end{align*}

By definition of $C_2$ at the start of Section \ref{s:nonsc}, we have 
$$\Ep{\p^*}{\|x\|_2^2}\leq C_2^2$$.

By Lemma \ref{l:nonstronglyconvexWassersteindistanceiscontractive}, we have
$$W_2(\p_{kh}, \ppi_h)\leq W_2(\p_0, \ppi_h)\leq C_1$$

By definition of $h^\prime$ at the start of Section \ref{s:nonsc}, and $h$ in Theorem \ref{t:nonstrongconvexityconvergence} (which ensures $h\leq h^\prime$), we have
$$W_2(\ppi_h,\p^*)\leq C_1$$

\end{proof}

\begin{proof}{Proof of Theorem \ref{t:nonstrongconvexityconvergence}}
First, we bound the discretization error (for an arbitrary $t$). By Lemma \ref{l:metricderivativefordiscretizationissmall}:
\begin{align*}
\at{\dds \left(F(\p_s) - F(\q_s^t)\right)}{s=t} \leq &\left(2L^2t \sqrt{\mathbb{E}_{\p_{\tau(t)}}{\|x_{\tau(t)}\|_2^2}} + 2L \sqrt{td}\right) \cdot \|\D_{\p_t}\|_*\\
\leq &\left(2L^2t \sqrt{\mathbb{E}_{\p_{\tau(t)}}{\|x_{\tau(t)}\|_2^2}} + 2L \sqrt{td}\right) \cdot \|\D_{\p_t}\|_*
\end{align*}

Given the choice of 
$$h= \frac{1}{48}\min\left\{\frac{\epsilon}{C_1(C_1+C_2)L^2}, \frac{\epsilon^2}{L^2 C_1^2d}, h^\prime\right\}$$,
we can ensure that
\begin{align*}
\left(L^2h \sqrt{E\|x_{\tau(t)}\|_2^2} + 2L \sqrt{hd}\right)\leq &\frac{1}{4}\left(L^2h \sqrt{18(C_1^2 + C_2^2)} + 2L \sqrt{hd}\right)\\
\leq &\frac{\epsilon}{8C_1}
\end{align*}

where the first inequality comes from Lemma \ref{l:vkvariancebound}.

Assume that $F(\p_s) - F(\p^*) \geq \epsilon$. By Lemma \ref{l:nonscupperboundbygradientdualnorm} and Corollary \ref{c:nonstronglyconvexWassersteindistancebound}, we have
\begin{align*}
\|\D_{\p_s}\|_* 
\geq & \frac{F(\p_s) - F(\p^*)}{W_2(\p_s,\p^*)}\\
\geq & \frac{\epsilon}{W_2(\p_s,\p^*)}\\
\geq & \frac{\epsilon}{4C_1} \numberthis \label{e:nonconvexgradientlowerbound}
\end{align*}
This implies that 
$$\at{\dds F(\p_s) - F(\q_s^t)}{s=t} \leq \frac{1}{2}\|\D_{\p_t}\|_*^2$$

The rate of decrease of $F(\p_t)$ thus satisfies
\begin{align*}
\ddt F(\p_t) - F(\p^*)=& \at{\ddt F(\q_s^t)-F(\p^*)}{s=t} + \at{\ddt(F(\p_s) - F(\q_s^t))}{s=t}\\
=& - \|\D_{\p_t}\|_*^2 + \frac{1}{2} \|\D_{\p_t}\|_*^2\\
\leq& - \frac{1}{2} \|\D_{\p_t}\|_*^2\\
\leq& - \frac{1}{2C_1^2} (F(\p_t) - F(\p^*))^2
\end{align*}
We now study two regimes. The first regime is when $F(\p_t) - F(\p^*)\geq 1$, $\ddt F(\p_t) - F(\p^*) \leq -\frac{1}{2C_1^2} (F(\p_t) - F(\p^*))$, which implies 
$$F(\p_t) - F(\p^*) \leq (F(\p_0) - F(\p^*))\exp(-\frac{t}{2C_1^2})$$
We thus achieve $F(\p_t) - F(\p^*) \leq 1$ in 
$$t\geq 2C_1^2 \log (F(\p_0) - F(\p^*))$$
In the second regime, $F(\p_t) - F(\p^*)\leq 1$. By noting that $f_t = \frac{1}{t}$ is the solution to $\ddt f_t = - f_t^2$, and letting $f_t = \frac{1}{2C_1^2}(F(\p_t) - F(\p^*))$, we get $F(\p_t) - F(\p^*) \leq \frac{2C_1^2}{t}$. To achieve $F(\p_t) - F(\p^*)\leq \epsilon$, we set $t= \frac{2C_1^2}{\epsilon}$. Overall, we just need to set
$$t\geq \frac{2C_1^2}{\epsilon} + 2C_1^2 \log (F(\p_0) - F(\p^*))$$
This, combined with the choice of $h$ earlier, proves the theorem.

\end{proof}

\subsection{Some regularity results}
In this subsection, we provide some regularity results needed in various parts of the paper.
\begin{Lemma}
\label{l:regularityofdensitiesandcurves}
Let $w_{\pmu}$ be as defined in Lemma \ref{l:firstvariationislinear}. Let $\p_t$ be as defined in \ref{e:discretelangevin}. For all $t$, $w_{\p_t}$ is well defined, and $\Ep{\p_t}{\|w_{\p_t}\|_2^2}$ is finite.
\end{Lemma}
\begin{proof}{Proof of Lemma \ref{l:regularityofdensitiesandcurves}}
First, we establish the following statement:
For any $t$, there exists a $\delta \in \Re$ with $\pmu_{\delta,y}(x)$ being the distribution of $N(y,\delta)$ and $\p\in \Pspace(\Re^d)$ such that 
\begin{enumerate}
\item For all $x\in \Re^d$, $\p_t(x)=\Ep{y\sim\p}{\pmu_{\delta,y}(x)}$
\item $\Ep{\p}{\|x\|_2^2}$ is finite.
\end{enumerate}
If $t=\tau(t)$, then let $\p = (Id(\cdot) - h \nabla U(\cdot))_\#\p_{\tau(t) -1}$ and let $\delta = 2h$. Otherwise, if $t \neq \tau(t)$, then let $\p = (Id(\cdot) - (t-\tau(t)) \nabla U(\cdot))_\#\p_{\tau(t)}$ and $\delta = 2(t-\tau(t))$. Where we used the definition of push-forward distribution from \eqref{s:definitions}. 1. now can be easily verified.

To see 2, let $t' = \tau(t) - 1$ in case 1 and let $t'=\tau(t)$ in case 2.
\begin{align*}
& \Ep{\p}{\|x\|_2^2}\\
=& \Ep{\p_{t'}}{\|x - h \nabla U(x)\|_2^2}\\
\leq& 2\Ep{\p_{t'}}{\|x\|_2^2}+2h^2\Ep{\p_{t'}}{\|\nabla U(x)\|_2^2}\\
\leq& 2\Ep{\p_{t'}}{\|x\|_2^2}+2h^2L^2\Ep{\p_{t'}}{\|x\|_2^2}\\
\leq& (2+2h^2L^2)\frac{4d}{m}
\end{align*}
Where the last inequality follows by Lemma \ref{l:strongconvexityboundexsq}.

Since $\pmu_{\delta,y}(x)$ for all $x,y$, $\Ep{\p}{\pmu_{\delta,y}(x)}$ is differentiable for all $x$. This proves the first part of the Lemma.

Next, a nice property of Gaussians is that 
$$\nabla_x \pmu_{\delta,y}(x) = -\frac{\pmu_{\delta,y}}{\delta}(x-y)$$
Thus, 
\begin{align*}
&\nabla \log \p_t(x)\\
=& \frac{1}{\delta} \nabla \log \Ep{y\sim \p}{\pmu_{\delta,y}(x)}\\
=& \frac{1}{\delta} \frac{1}{\Ep{y\sim \p}{\pmu_{\delta,y}(x)}} \Ep{y\sim \p}{\lrp{(y-x)\pmu_{\delta,y}(x))}}\\
=& \frac{1}{\delta }\Ep{y\sim \pmu^{x}_\delta}{y} -x
\end{align*}
Where $\pmu^{x}_\delta$ denotes the conditional distribution of $y$ given $x$, when $y\sim \p$ and $x\sim \pmu_{\delta,y}$.

Thus
\begin{align*}
& \Ep{x\sim \p_t(x)}{\|\nabla \log \p_t(x)\|_2^2} \\
\leq & \frac{1}{\delta}\Ep{x\sim \p_t(x)}{2\Ep{y\sim \pmu^{x}_\delta}{\|y\|_2^2} + 2\|x\|_2^2}\\
= & \frac{2}{\delta}\Ep{y\sim \p}{\|y\|_2^2} + \frac{2}{\delta}\Ep{x\sim \p_t} {\|x\|_2^2}\\
\leq & \infty
\end{align*}
Where the first inequality is by Jensen's inequality and Young's inequality and the preceding result, the second inequality is by definition of conditional distribution, the third inequality is by the fact that $\delta>0$ (by definition at the start of the proof), the fact that $\Ep{x\sim \p_t} {\|x\|_2^2}\leq \frac{4d}{m}$ (by Lemma \ref{l:strongconvexityboundexsq}), and by the fact that $\Ep{y\sim \p}{\|y\|_2^2}< \infty$ (see item 2. at the start of the proof)

Finally, we have that 
\begin{align*}
&\|w_{\p_t}\|_\Lp{\p_t}^2\\
=&\Ep{\p_t}{\|w_{\p_t}(x)\|_2^2}\\
=&\Ep{\p_t}{\|\nabla \log \p_t(x) - \nabla \log \p^*(x)\|_2^2}\\
\leq & 2\Ep{\p_t}{\|\nabla \log \p_t(x)\|_2^2}+2\Ep{\p_t}{\|\nabla \log \p^*(x)\|_2^2}\\
< & \infty
\end{align*}
Where the last inequality uses the fact that $\|\nabla \log \p^(x)\|_2=\|\nabla U(x)\|_2\leq L\|x\|_2$ and $\Ep{\p_t}{\|x\|_2^2}\leq \frac{4d}{m}$.
\end{proof}

\begin{Lemma}
\label{l:pthasboundedmetricderivative}
Let $\p_t$ be as defined in \eqref{e:discretelangevin}. Then $|\p_t'|$ is finite for all $t$, where $|\p_t'|$ is the metric derivative of $\p_t$, as defined in \eqref{d:metricderivative}.
\end{Lemma}
\begin{proof}{Proof of Lemma \ref{l:pthasboundedmetricderivative}}
We define the random variable $\xi$ to be distributed as $N(0,1)$.
For all $t$, let $\x_t$ be as defined in \eqref{e:discretelangevin}.
One can verify that the random variable $y_t \triangleq x_{\tau(t)} - \nabla (t-\tau(t))U(x_{\tau(t)}) + \sqrt{2(t-\tau(t))}\xi$ has the same distribution as $x_t$. Thus $y_t$ and $y_{t+\epsilon}$ define a coupling between $\p_t$ and $\p_{t+\epsilon}$. We thus have

Let $h\triangleq t-\tau(t)$
\begin{align*}
&|\p_t'|\\
=& \lim_{\epsilon \to 0} \frac{1}{\epsilon} W_2(\p_t, \p_{t+\epsilon})\\
\leq & \lim_{\epsilon \to 0} \frac{1}{\epsilon} \sqrt{\E{\|y_t - y_{t+\epsilon}\|_2^2}}\\
= & \lim_{\epsilon \to 0} \frac{1}{\epsilon} \sqrt{\Ep{x\sim\p_{\tau(t)}}{\|\epsilon \nabla U(x) + (\sqrt{2(h + \epsilon)} - \sqrt{2h}) \xi\|_2^2}}\\
= & \lim_{\epsilon \to 0} \frac{1}{\epsilon} \sqrt{\Ep{x\sim\p_{\tau(t)}}{\|\epsilon \nabla U(x)\|_2^2} + \E{\|(\sqrt{2(h + \epsilon)} - \sqrt{2h}) \xi\|_2^2}}\\
\leq & \lim_{\epsilon \to 0} \frac{1}{\epsilon} \sqrt{\Ep{x\sim\p_{\tau(t)}}{\|\epsilon \nabla U(x)\|_2^2}} \\
& \quad + \frac{1}{\epsilon}\sqrt{ \E{\|(\sqrt{2(h + \epsilon)} - \sqrt{2h}) \xi\|_2^2}}\\
= & \sqrt{\Ep{x\sim\p_{\tau(t)}}{\| \nabla U(x)\|_2^2}} + \frac{1}{\sqrt{8h}} \sqrt{\E{\|\xi\|_2^2}}\\
\end{align*}
Where the last inequality follows by Taylor expansion of $\sqrt{2h + 2\epsilon}$. We can bound the first term by a finite number using $\|\nabla U(x)\|_2^2 \leq L^2 \|x\|_2^2$, then applying Lemma \ref{l:strongconvexityboundexsq}. The second  term is finite for $h\neq 0$.

For the case $h=0$, we know that $w_{\p_t}$ satisfies the continuity equation for $\p_t$ at $t$, and so $|\p_t'|= \|w_{\p_t}\|_{\Lp{\p_t}}<\infty$, by Lemma \ref{l:metricderivativeandvnorm} and Lemma \ref{l:regularityofdensitiesandcurves}.

\end{proof}

\end{document}